\documentclass{article}%
\usepackage{amsmath}
\usepackage{amsfonts}
\usepackage{amssymb}
\usepackage{graphicx}
\usepackage{algorithmic}
\usepackage{algorithm}%
\providecommand{\U}[1]{\protect\rule{.1in}{.1in}}

\newtheorem{proposition}{Proposition}
\newenvironment{proof}[1][Proof]{\textbf{#1.} }{\  \rule{0.5em}{0.5em}}
\begin{document}

\title{Attention and Self-Attention in Random Forests}
\author{Lev V. Utkin and Andrei V. Konstantinov\\Peter the Great St.Petersburg Polytechnic University\\St.Petersburg, Russia\\e-mail: lev.utkin@gmail.com, andrue.konst@gmail.com}
\date{}
\maketitle

\begin{abstract}
New models of random forests jointly using the attention and self-attention
mechanisms are proposed for solving the regression problem. The models can be
regarded as extensions of the attention-based random forest whose idea stems
from applying a combination of the Nadaraya-Watson kernel regression and the
Huber's contamination model to random forests. The self-attention aims to
capture dependencies of the tree predictions and to remove noise or anomalous
predictions in the random forest. The self-attention module is trained jointly
with the attention module for computing weights. It is shown that the training
process of attention weights is reduced to solving a single quadratic or
linear optimization problem. Three modifications of the general approach are
proposed and compared. A specific multi-head self-attention for the random
forest is also considered. Heads of the self-attention are obtained by
changing its tuning parameters including the kernel parameters and the
contamination parameter of models. Numerical experiments with various datasets
illustrate the proposed models and show that the supplement of the
self-attention improves the model performance for many datasets.


\end{abstract}

\section{Introduction}

The attention mechanism is an effective method for improving the performance
of neural networks. It was proposed to enhance the natural language processing
models, and later becomes one of the most effective methods for various
machine learning tasks. A neural network with attention-based components can
automatically distinguish the relative importance of features or examples by
means of assigning the corresponding weights to them to get a higher
classification or regression accuracy. The main idea behind the attention
mechanism stems from the human perception property to concentrate on an
important part of information and to ignore other information
\cite{Niu-Zhong-Yu-21}. Due to the ability of the attention to significantly
improve the neural network performance, a huge amount of attention-based
models have been developed to be used in various applications and tasks
\cite{Niu-Zhong-Yu-21,Chaudhari-etal-2019,Correia-Colombini-21a,Correia-Colombini-21,Lin-Wang-etal-21}%
.

In spite of success of many neural attention models in solving various
application tasks, attention is a component of neural architectures
\cite{Chaudhari-etal-2019}. This implies that the attention weights are
learned by incorporating an additional feed forward neural network within the
architectures. The corresponding models meet difficulties of neural networks,
including, overfitting, many tuning parameters, requirements of a large amount
of data, the black-box nature, expensive computations. Moreover, tabular
learning data may be also an important problem encountered with neural
networks. There are several deep learning models
\cite{Arik-Pfister-21,Katzir-etal-21,Somepalli-etal-21} illustrating
efficiency on some tabular datasets. However, other experimental studies
\cite{Borisov-etal-21,Shwartz-Amitai-22} give opportunity to conclude that
ensemble-based models using decision trees as weak learners mainly outperform
deep neural networks when they deal with heterogeneous tabular data.

Taking the above into account, Utkin and Konstantinov
\cite{Konstantinov-Utkin-22d,Utkin-Konstantinov-22a} proposed a new model
called the attention-based random forest (ABRF) by incorporating the attention
mechanism into ensemble-based models such as random forests (RF)
\cite{Breiman-2001} and the gradient boosting machine
\cite{Friedman-2001,Friedman-2002}. The original RF is a powerful model which
consists of a large number of randomly built individual decision trees such
that their predictions are combined, for example, by means of the simple
averaging. Decision trees are built by the random selection of different
subsamples of examples from training data and different subsamples of the
feature space. The main idea behind the proposed ABRF models stems from the
Nadaraya-Watson kernel regression model \cite{Nadaraya-1964,Watson-1964}, but
attention weights used in the Nadaraya-Watson regression are assigned to
decision trees in a specific way. The weights can be regarded as the attention
weights because they are defined by using queries, keys and values concepts in
terms of the attention mechanism. In contrast to weights of trees defined in
\cite{Utkin-Kovalev-Coolen-2020,Utkin-Kovalev-Meldo-2019}, weights in ABRF
have trainable parameters and depend on how far an example, which falls into a
leaf of a tree, is from examples which fall into the same leaf. The resulting
prediction of ABRF is computed as a weighted sum of the tree predictions.

Three modifications of ABRF were studied in
\cite{Konstantinov-Utkin-22d,Utkin-Konstantinov-22a}. The first modification
called $\epsilon$-ABRF uses the Huber's $\epsilon$-contamination model
\cite{Huber81} for defining the attention weights. Each weight consists of two
parts: the softmax operation with the tuning coefficient $1-\epsilon$ and the
trainable bias of the softmax weight with coefficient $\epsilon$. An important
advantage of $\epsilon$-ABRF is that attention weights of trees linearly
depend on trainable parameters. This property leads to solving the standard
quadratic optimization problem which is simply solved. $\epsilon$-ABRF avoids
using the gradient-based algorithm to compute optimal trainable parameters.
Other two ABRF modifications differ from $\epsilon$-ABRF by the additional
trainable attention parameters incorporated into the softmax operation. These
modifications require to apply the gradient-based algorithms to compute
optimal attention parameters.

$\epsilon$-ABRF has demonstrated outperforming results for many real datasets.
However, an \textquotedblleft unfortunate\textquotedblright\ selection of a
subset of training examples for building a tree can lead to anomalous or
incorrect predictions which bias the RF prediction. In order to overcome this
disadvantage and following the idea behind $\epsilon$-ABRF, we propose to
supplement this model by the self-attention mechanism \cite{Vaswani-etal-17}
which aims to capture dependencies of the tree predictions and to remove noise
or anomalous predictions in $\epsilon$-ABRF. The proposed model is called
SAT-RF (self-attention-attention-based RF) The main peculiarity of the
supplemented self-attention is that it is trained jointly with the attention
mechanism, but not sequentially, i.e., we solve a single optimization problem
for simultaneous computing trainable parameters of $\epsilon$-ABRF and the
supplemented self-attention. As a result, both the mechanisms impact each
other and can be regarded as a joint attention-based modification. The use of
the Huber's $\epsilon$-contamination model with parameter $\epsilon$ different
from $\epsilon$-ABRF for defining the self-attention leads to the quadratic or
linear optimization problem with trainable parameters of $\epsilon$-ABRF and
the supplemented self-attention as optimization variables. This is an
important property of the proposed model. Moreover, we proposed a specific
variant of the multi-head self-attention which allows us to combine knowledge
of the self-attention via different representation of its tuning parameters.
It is shown that the multi-head self-attention also leads to the quadratic or
linear optimization problem for computing trainable parameters of all heads.

Our contributions can be summarized as follows:

\begin{enumerate}
\item A new attention-based RF model is proposed. According to the model, the
trainable self-attention mechanism is incorporated into the attention-based RF
as an additional component to capture dependencies of the tree predictions and
to remove noise or anomalous predictions in $\epsilon$-ABRF. It is important
that the self-attention and attention components are jointly trained such that
trainable parameters of the attention impact on parameters of the
self-attention and vice versa.

\item Three modifications of the self-attention is studied. The first one
(SAT-RF-y), is based on comparison of predictions provided by pairs of trees.
The second modification (SAT-RF-x) takes into account only distances between
mean feature vectors which are determined from all feature vectors which fall
into the same leaves with the tested example in pairs of trees. The third
modification (SAT-RF-yx) can be regarded as a combination of the first and the
second modifications.

\item A specific multi-head self-attention for the RF is proposed. Heads are
obtained by changing the tuning parameters of the self-attention. They are
trained by solving a single quadratic optimization problem for computing the
optimal attention and self-attention weights. It can be said that the whole
model is trained \textquotedblleft end-to-end\textquotedblright\ to some extent.

\item Various numerical experiments with real tabular datasets are provided to
justify SAT-RF, to study its peculiarities and to compare it with original
RFs. Moreover, we investigate two types of RFs: original RFs and Extremely
Randomized Trees (ERT). At each node, the ERT algorithm chooses a split point
randomly for each feature and then selects the best split among these
\cite{Geurts-etal-06}.
\end{enumerate}

The paper is organized as follows. Related work can be found in Section 2. A
brief introduction to the attention and self-attention mechanisms is given in
Section 3. A general approach to incorporating the attention and the
self-attention into the RF is provided in Section 4. Analysis of the attention
and self-attention representations by means of the Huber's $\epsilon
$-contamination model is given in Section 5. Some questions of applying the
multi-head self-attention in the framework of the general approach is
considered in Section 6. Numerical experiments with real data illustrating the
efficiency of the proposed models for solving the regression problems are
provided in Section 7. Concluding remarks can be found in Section 8.

\section{Related work}

\textbf{Attention mechanism}. Many attention-based models have been developed
to improve the performance of classification and regression algorithms.
Surveys of various attention-based models are available in
\cite{Niu-Zhong-Yu-21,Chaudhari-etal-2019,Correia-Colombini-21a,Correia-Colombini-21,Lin-Wang-etal-21,Liu-Huang-etal-21}%
.

It should be noted that one of the computational problems of attention
mechanisms is training through the softmax function. In order to overcome this
difficulty, several interesting approaches have been proposed. Choromanski et
al. \cite{Choromanski-etal-21} introduced Performers as a Transformer
architecture which can estimate softmax attention with provable accuracy using
only linear space and time complexity. A linear unified nested attention
mechanism that approximates softmax attention with two nested linear attention
functions was proposed by Ma et al. \cite{Ma-Kong-etal-21}. A new class of
random feature methods for linearizing softmax and Gaussian kernels called
hybrid random features was introduced in \cite{Choromanski-etal-21a}. The same
problem is solved in \cite{Peng-Pappas-etal-21} where the authors propose
random feature attention, a linear time and space attention that uses random
feature methods to approximate the softmax function. Schlag et al.
\cite{Schlag-etal-2021} proposed a new kernel function to linearize attention
which balances simplicity and effectiveness. A detailed survey of techniques
of random features to speed up kernel methods was provided by Liu et al.
\cite{Liu-Huang-etal-21}.

\textbf{Self-attention}. The self-attention was proposed by Vaswani et al.
\cite{Vaswani-etal-17} as an important component of a new neural network
architecture known as Transformer. It is inspired by the previous works
presented by Cheng et al. \cite{ChengDong-Lapata-16}, where self-attention is
called intra-attention, by Parikh et al. \cite{Parikh-etal-16}. The
self-attention aims to capture token dependencies and to relate distinct
positions in the input sequence. It has been used in many tasks, for example,
sentence embedding \cite{Lin-Feng-etal-17}, in machine translation and natural
language processing \cite{Dai-Yang-etal-19,Devlin-etal-18,Wu-Fan-etal-19}, in
speech recognition \cite{Povey-etal-18,Shim-Choi-Sung-22,Vyas-etal-20}, in
image recognition
\cite{Chen-Xie-etal-20,Guo-Liu-etal-21,Khan-Naseer-etal-22,Liu-Lin-etal-21,Ramachandran-etal-19,Shen-Bello-etal-20,Wang-Jiang-etal-17,Wang-Girshick-etal-18,Zhao-Jia-Koltun-20}%
.

Many survey papers have been devoted to various aspects and applications of
attention and self-attention mechanisms, for example,
\cite{Chaudhari-etal-2019,Lin-Wang-etal-21,Khan-Naseer-etal-22,Brauwers-Frasincar-22,Goncalves-etal-2022,Hassanin-etal-2022,Santana-Colombini-21,Soydaner-22,Xu-Wei-etal-22}%
.

We use self-attention to remove anomalies in the tree predictions. Similar
approaches to image denoising were considered in
\cite{Li-Hsu-etal-20,Tian-Fei-etal-20,Vidal-22,Yu-Nie-etal-21,Zuo-Chen-etal-22}%
.

It should be noted that the above methods are implemented as a part of a
neural network, and they are not studied for application to other machine
learning models, for example, to RFs.

\textbf{Weighted RFs}. Many models were developed and studied to incorporate
weights of trees into RFs. They can be divided into two groups. Models from
the first group are based on assigning weights to decision trees in accordance
with some criteria to improve the classification and regression models
\cite{Kim-Kim-Moon-Ahn-2011,Li-Wang-Ding-Dong-2010,Ronao-Cho-2015,Winham-etal-2013,Xuan-etal-18,Zhang-Wang-21}%
. For example, a model proposed in \cite{Daho-2014} uses weights of classes to
deal with imbalanced datasets. However, the assigned weights in models from
the first group are not trainable parameters. They can be viewed as tuning
parameters. Attempts to train weights of trees were carried out in
\cite{Utkin-Kovalev-Coolen-2020,Utkin-Kovalev-Meldo-2019,Utkin-Konstantinov-etal-20,Utkin-etal-2019}%
, where weights are assigned by solving optimization problems, i.e., they are
incorporated into a certain loss function of the whole RF such that the loss
function is minimized over values of weights. Another approach was proposed in
\cite{Utkin-Konstantinov-22a}. In contrast to the aforementioned models, this
approach is based on using the attention mechanism and weights assigned to
trees depend not only on trees, but on each example. These weights can be
regarded as attention weights. Similar attention-based model for the gradient
boosting machine was proposed in \cite{Konstantinov-Utkin-22d}

\section{Preliminaries}

\subsection{Attention mechanism as the Nadaraya-Watson regression}

The attention mechanism can be regarded as a tool by which a neural network
can automatically distinguish the relative importance of features and weigh
the features for enhancing the classification accuracy. It can be viewed as a
learnable mask which emphasizes relevant information in a feature map. It is
pointed out in \cite{Chaudhari-etal-2019,Zhang2021dive} that the original idea
of attention can be understood from the statistical point of view applying the
Nadaraya-Watson kernel regression model \cite{Nadaraya-1964,Watson-1964}.

Given $n$ examples $S=\{(\mathbf{x}_{1},y_{1}),(\mathbf{x}_{2},y_{2}%
),...,(\mathbf{x}_{n},y_{n})\}$, in which $\mathbf{x}_{i}=(x_{i1}%
,...,x_{im})\in\mathbb{R}^{m}$ represents a feature vector involving $m$
features and $y_{i}\in\mathbb{R}$ represents the regression outputs, the task
of regression is to construct a regressor $f:\mathbb{R}^{m}\rightarrow
\mathbb{R}$ which can predict the output value $\tilde{y}$ of a new
observation $\mathbf{x}$, using available data $S$. The similar task can be
formulated for the classification problem.

The original idea behind the attention mechanism is to replace the simple
average of outputs $\tilde{y}=n^{-1}\sum_{i=1}^{n}y_{i}$ for estimating the
regression output $y$, corresponding to a new input feature vector
$\mathbf{x}$ with the weighted average, in the form of the Nadaraya-Watson
regression model \cite{Nadaraya-1964,Watson-1964}:%
\begin{equation}
\tilde{y}=\sum_{i=1}^{n}\alpha(\mathbf{x},\mathbf{x}_{i})y_{i},
\end{equation}
where weight $\alpha(\mathbf{x},\mathbf{x}_{i})$ conforms with relevance of
the $i$-th example to the vector $\mathbf{x}$.

According to the Nadaraya-Watson regression model, to estimate the output $y$
for an input variable $\mathbf{x}$, training outputs $y_{i}$ given from a
dataset weigh in agreement with the corresponding input $\mathbf{x}_{i}$
locations relative to the input variable $\mathbf{x}$. The closer an input
$\mathbf{x}_{i}$ to the given variable $\mathbf{x}$, the greater the weight
assigned to the output corresponding to $\mathbf{x}_{i}$.

One of the original forms of weights is defined by a kernel $K$ (the
Nadaraya-Watson kernel regression \cite{Nadaraya-1964,Watson-1964}), which can
be regarded as a scoring function estimating how vector $\mathbf{x}_{i}$ is
close to vector $\mathbf{x}$. The weight is written as follows:
\begin{equation}
\alpha(\mathbf{x},\mathbf{x}_{i})=\frac{K(\mathbf{x},\mathbf{x}_{i})}%
{\sum_{j=1}^{n}K(\mathbf{x},\mathbf{x}_{j})}.
\end{equation}

In particular, If to use the Gaussian kernel, then weights are of the form:
\begin{equation}
\alpha(\mathbf{x},\mathbf{x}_{i})=\text{\textrm{softmax}}\left(
-\frac{\left\Vert \mathbf{x}-\mathbf{x}_{i}\right\Vert ^{2}}{2\tau}\right)  ,
\end{equation}
where $\tau$ is the tuning parameter.

In terms of the attention mechanism \cite{Bahdanau-etal-14}, vector
$\mathbf{x}$, vectors $\mathbf{x}_{i}$ and outputs $y_{i}$ are called as the
\textit{query}, \textit{keys} and \textit{values}, respectively. Weight
$\alpha(\mathbf{x},\mathbf{x}_{i})$ is called as the attention weight.

Generally, weights $\alpha(\mathbf{x},\mathbf{x}_{i})$ can be extended by
incorporating trainable parameters.

Several definitions of attention weights and attention mechanisms have been
proposed. The most popular definitions are the additive attention
\cite{Bahdanau-etal-14}, the multiplicative or dot-product attention
\cite{Luong-etal-2015,Vaswani-etal-17}.

\subsection{Self-attention mechanism as the non-local means denoising}

One of the interesting interpretations of the self-attention mechanism is the
non-local means denoising \cite{Vidal-22}, which aims to remove noise in an
image by computing average intensity of each pixel from a set of neighboring
pixels. This idea again stems from the Nadaraya-Watson regression under
condition that the query is a key, and each key coincides with the
corresponding value.

According to \cite{Vidal-22}, intensity of a pixel with coordinates
$\mathbf{x}$ by using the non-local means denoising is determined as follows:
\begin{equation}
f(\mathbf{x})=y^{\ast}=\sum_{i=1}^{n}\beta(y,y_{i})y_{i},
\end{equation}
where weight $\beta(y,y_{i})$ is determined as
\begin{equation}
\beta(y,y_{i})=\frac{K(y,y_{i})}{\sum_{j=1}^{n}K(y,y_{j})}.
\end{equation}

If to use the Gaussian kernel, then weights are of the form:
\begin{equation}
\beta(y,y_{i})=\text{\textrm{softmax}}\left(  -\frac{\left(  y-y_{i}\right)
^{2}}{2\kappa}\right)  ,
\end{equation}
where $\kappa$ is the tuning or training parameter.

SAT-RF with the above definition of the softmax operation is called SAT-RF-y.
Generally, the query $f(\mathbf{x})$ and values $y_{i}$ can be vectors.
Moreover, the values can be taken in another form. Variants of the forms are
considered below.

\section{Self-attention-based random forest}

The regression problem is to construct a regression function $f$ such that
$y_{i}=f(\mathbf{x}_{i},\theta)+\xi$, where $\xi$ is the random noise with
expectation $0$ and a finite variance; $\theta$ is a set of trainable
parameter; $(\mathbf{x}_{i},y_{i})$ is the $i$-th example from the training
set $S$, $i=1,...,n$. In a simple case, $f(\mathbf{x,}\theta)$ minimizes the
expected error, for example, $n^{-1}\sum_{i=1}^{n}\left(  y_{i}-f(\mathbf{x}%
_{i}\mathbf{,}\theta)\right)  ^{2}$ over $\theta$.

One of the powerful machine learning models handling with tabular data is the
RF which can be regarded as an ensemble of $T$ decision trees such that each
tree is trained on a subset of examples randomly selected from the training
set. In the original RF, the final RF prediction $\tilde{y}$ for a testing
example $\mathbf{x}$ is determined by averaging predictions $\tilde{y}%
_{1},...,\tilde{y}_{T}$ obtained for all trees.

Denote an index set of examples which fall into the $i$-th leaf in the $k$-th
tree as $\mathcal{J}_{i}^{(k)}$ such that $\mathcal{J}_{i}^{(k)}%
\cap\mathcal{J}_{j}^{(k)}=\varnothing$ because the same example cannot fall
into different leaves of the same tree. Let us consider an example
$\mathbf{x}$ which falls into $i$-th leaf in the $k$-th tree. Then we can
introduce the mean vector $\mathbf{A}_{k}(\mathbf{x)}$ defined as the mean of
training vectors $\mathbf{x}_{j}$ which fall into the $i$-th leaf of the
$k$-th tree, i.e., into the leaf where vector $\mathbf{x}$ felt into. In the
same way, we introduce the mean target value $B_{k}(\mathbf{x)}$ defined as
the mean of $y_{j}$ such that $j\in\mathcal{J}_{i}^{(k)}$. In fact, value
$B_{k}(\mathbf{x)}$ in regression coincides with the prediction of the $k$-th
tree. Formally, we write
\begin{equation}
\mathbf{A}_{k}(\mathbf{x)}=\frac{1}{\#\mathcal{J}_{i}^{(k)}}\sum
_{j\in\mathcal{J}_{i}^{(k)}}\mathbf{x}_{j}, \label{RF_Att_20}%
\end{equation}%
\begin{equation}
B_{k}(\mathbf{x)}=\frac{1}{\#\mathcal{J}_{i}^{(k)}}\sum_{i\in\mathcal{J}%
_{j}^{(k)}}y_{j}. \label{RF_Att_21}%
\end{equation}

By returning to the Nadaraya-Watson regression and notation of the attention
mechanism framework, the set of $\mathbf{A}_{k}(\mathbf{x)}$, $k=1,...,T$, can
be regarded as a set of keys for every $\mathbf{x}$, the set of $B_{k}%
(\mathbf{x)}$ can be regarded as a set of values. This implies that the final
prediction $\tilde{y}$ of the RF can be computed by using the Nadaraya-Watson
regression, namely,
\begin{equation}
\tilde{y}=f(\mathbf{x},\mathbf{w})=\sum_{k=1}^{T}\alpha\left(  \mathbf{x}%
,\mathbf{A}_{k}(\mathbf{x)},w_{k}\right)  \cdot B_{k}(\mathbf{x)}.
\label{RF_Att_47}%
\end{equation}

Here $\alpha\left(  \mathbf{x},\mathbf{A}_{k}(\mathbf{x)},w_{k}\right)  $ is
the attention weight with vector $\mathbf{w}=(w_{1},...,w_{T})$ of trainable
parameters assigned to the $k$-th tree. One can see that the set of parameters
$\theta$ is replaced with $\mathbf{w}$. If $\alpha$ is the normalized kernel,
then it is defined through the distance between $\mathbf{x}$ and
$\mathbf{A}_{k}(\mathbf{x)}$, which is defined, for instance, by means of
$L_{2}$-norm $\left\Vert \mathbf{x}-\mathbf{A}_{k}(\mathbf{x)}\right\Vert
^{2}$. It is assumed that
\begin{equation}
\sum_{k=1}^{T}\alpha\left(  \mathbf{x},\mathbf{A}_{k}(\mathbf{x)}%
,w_{k}\right)  =1, \label{RF_Att_101}%
\end{equation}%
\begin{equation}
\sum_{k=1}^{T}w_{k}=1. \label{RF_Att_102}%
\end{equation}

Condition (\ref{RF_Att_101}) is due to properties of the attention weights in
the Nadaraya-Watson regression. Condition (\ref{RF_Att_102}) is explained
below when the Huber's $\epsilon$-contamination model will be considered for
representing the attention weights.

The above approach to incorporating the attention mechanism into the RF has
been proposed in \cite{Konstantinov-Utkin-22d,Utkin-Konstantinov-22a}. Our aim
now is to supplement it with the self-attention.

We suppose that there may be anomalies among values $B_{k}(\mathbf{x})$ or
$\tilde{y}_{k}$. In order to cope with the anomalies, we apply the
self-attention mechanism which corrects every $\tilde{y}_{k}$. According to
the self-attention, each $\tilde{y}_{i}$ can be recalculated as follows:
\begin{equation}
y_{j}^{\ast}=\sum_{i=1}^{T}\beta\left(  \tilde{y}_{j},\tilde{y}_{i}%
,v_{i}\right)  \cdot\tilde{y}_{i}. \label{RF_Att_121}%
\end{equation}
\bigskip

Here $\beta\left(  \tilde{y}_{i},\tilde{y}_{k},v_{k}\right)  $ is the
self-attention weight with vector $\mathbf{v}=(v_{1},...,v_{T})$ of trainable
parameters assigned to the $k$-th tree such that $\sum_{k=1}^{T}v_{k}=1$.

The main idea behind the approach is to use the attention and self-attention
simultaneously. Let us substitute (\ref{RF_Att_121}) into (\ref{RF_Att_47})
under condition $B_{k}(\mathbf{x)=}\tilde{y}_{k}$ as
\begin{equation}
\tilde{y}=f(\mathbf{x},\mathbf{w},\mathbf{v})=\sum_{i=1}^{T}\sum_{k=1}%
^{T}\alpha\left(  \mathbf{x},\mathbf{A}_{i}(\mathbf{x)},w_{i}\right)
\cdot\beta\left(  \tilde{y}_{i},\tilde{y}_{k},v_{k}\right)  \cdot\tilde{y}%
_{k}. \label{RF_Att_123}%
\end{equation}

We get the trainable attention-based RF with parameters $\mathbf{w}$ and
$\mathbf{v}$, which are defined by minimizing the expected loss function over
set $\mathcal{W}$ and set $\mathcal{V}$ of parameters, respectively, as
follows:
\begin{equation}
(\mathbf{w}_{opt},\mathbf{v}_{opt})=\arg\min_{\mathbf{w\in}\mathcal{W}%
,\ \mathbf{v}\in\mathcal{V}}~\sum_{s=1}^{n}L\left(  \tilde{y}_{s}%
,y_{s},\mathbf{w},\mathbf{v}\right)  . \label{RF_Att_49}%
\end{equation}

The loss function can be rewritten as
\begin{align}
&  \sum_{s=1}^{n}L\left(  \tilde{y}_{s},y_{s},\mathbf{w},\mathbf{v}\right)
\nonumber\\
&  =\sum_{s=1}^{n}\left(  y_{s}-\sum_{i=1}^{T}\sum_{k=1}^{T}\alpha\left(
\mathbf{x},\mathbf{A}_{i}(\mathbf{x)},w_{i}\right)  \cdot\beta\left(
\tilde{y}_{i},\tilde{y}_{k},v_{k}\right)  \cdot\tilde{y}_{k}\right)  ^{2}.
\label{RF_Att_50}%
\end{align}

Optimal trainable parameters $\mathbf{w},\mathbf{v}$ are computed depending on
forms of attention weights $\alpha$ and self-attention weights $\beta$.
Moreover, the computation time for solving the optimization problem
(\ref{RF_Att_50}) also significantly depends on the weights. Therefore, we
propose the form which leads to convex quadratic optimization problem.

It can be seen from the above that every value $y$ is transformed to $y^{\ast
}$ in accordance with the difference between $y$ and other values $y_{i}$.
However, the above non-local means denoising does not take into account the
distance between the vectors $A_{i}(\mathbf{x})$ and $A_{j}(\mathbf{x})$. In
other words, it is interesting to take into account how the mean feature
vector of all feature vectors which fall into the same leaves with
$\mathbf{x}$ of the $i$-th and the $j$-th trees, respectively. Hence, we can
write the self-attention weight as
\begin{equation}
\beta(A_{i},A_{j})=\text{\textrm{softmax}}\left(  -\frac{\left\Vert
A_{i}(\mathbf{x})-A_{j}(\mathbf{x})\right\Vert ^{2}}{2\kappa}\right)  ,
\end{equation}

SAT-RF with the above definition of the softmax operation is called SAT-RF-x.

By intuition, if $A_{i}(\mathbf{x})$ and $A_{j}(\mathbf{x})$ are close to each
other, then we can expect that the difference between values $y_{i}$ and
$y_{j}$ is small. If it is large, then the weight of $y_{i}$ should be larger
than in the case when the difference between values $y_{i}$ and $y_{j}$ is
small. On the contrary, if $A_{i}(\mathbf{x})$ and $A_{j}(\mathbf{x})$ are far
from each other, then the impact of value $y_{j}$ is reduced and the
corresponding weight should be decreased even if the difference between values
$y_{i}$ and $y_{j}$ is small. The above reasoning leads to applying the
following self-attention weights:
\begin{equation}
\beta(y_{i},y_{j})=\text{\textrm{softmax}}\left(  -\frac{\left(  y_{i}%
-y_{j}\right)  ^{2}}{2\kappa\left\Vert A_{i}(\mathbf{x})-A_{j}(\mathbf{x}%
)\right\Vert ^{2}}\right)  ,
\end{equation}

SAT-RF with the above definition of the softmax operation is called SAT-RF-yx.

It should be pointed out that the modifications of SAT-RF do not impact on the
general approach, and they define only the softmax operations. Therefore, all
expressions will be given using the first modifications, but results of
numerical experiments will be considered for every modification.

\section{Self-attention and the Huber's contamination model}

To simplify computations and to get a unique solution for $\mathbf{w}$, we
propose to use the well-known Huber's $\epsilon$-contamination model
\cite{Huber81} which can be represented as follows:%
\begin{equation}
(1-\epsilon)\cdot P+\epsilon\cdot Q, \label{RF_Att_40}%
\end{equation}
where the probability distribution $P$ is contaminated by some arbitrary
distribution $Q$; the rate $\epsilon\in\lbrack0,1]$ is a model parameter which
control the size of the solution set.

The use of the $\epsilon$-contamination model stems from several reasons.
First of all, the softmax function can be interpreted as the probability
distribution $P$ in (\ref{RF_Att_40}) because its sum is $1$. It can be
represented as a point in the probabilistic unit simplex having $T$ vertices.
Second, weights $\alpha\left(  \mathbf{x},\mathbf{A}_{i}(\mathbf{x)}%
,w_{i}\right)  $ also can be interpreted as a probability distribution or
another point in the same unit simplex. This point is biased by means of the
probability distribution $Q$ in (\ref{RF_Att_40}) which is trained in order to
achieve the best prediction results. The contamination parameter $\epsilon$
can be regarded as a tuning parameter of the model. It should be noted that
$\epsilon$ can be viewed as the trainable parameter. However, this case leads
to a more complex optimization problem. After substituting elements of
$\alpha\left(  \mathbf{x},\mathbf{A}_{i}(\mathbf{x)},w_{i}\right)  $ into
(\ref{RF_Att_40}), we get
\begin{equation}
\alpha\left(  \mathbf{x}_{s},\mathbf{A}_{i}(\mathbf{x)},w_{i}\right)
=(1-\epsilon)\cdot\text{\textrm{softmax}}\left(  \left\Vert \mathbf{x}%
_{s}-\mathbf{A}_{i}(\mathbf{x}_{s}\mathbf{)}\right\Vert ^{2}/\tau\right)
+\epsilon\cdot w_{i}. \label{RF_Att_41}%
\end{equation}

Let us define the self-attention weights $\beta\left(  \tilde{y}_{i},\tilde
{y}_{k},v_{k}\right)  $ in the same way using the Huber's $\gamma
$-contamination model. In this case, we can write the similar expression:
\begin{equation}
\beta\left(  \tilde{y}_{i},\tilde{y}_{k},v_{k}\right)  =(1-\gamma
)\cdot\text{\textrm{softmax}}\left(  \left(  \tilde{y}_{i}-\tilde{y}%
_{k}\right)  ^{2}/\kappa\right)  +\gamma\cdot v_{k}. \label{RF_Att_141}%
\end{equation}

Here $\gamma$ is the same parameter of the contamination model as $\epsilon$.

After substituting (\ref{RF_Att_41}) and (\ref{RF_Att_141}) into
(\ref{RF_Att_123}), we get
\begin{equation}
\sum_{i=1}^{T}\sum_{k=1}^{T}\left(  D_{si}+\epsilon w_{i}\right)  \cdot\left(
C_{ik}+\gamma v_{k}\right)  \cdot\tilde{y}_{k}, \label{RF_Att_144}%
\end{equation}
where
\begin{equation}
D_{sk}=(1-\epsilon)\cdot\mathrm{softmax}\left(  \left\Vert \mathbf{x}%
_{s}-\mathbf{A}_{i}(\mathbf{x}_{s}\mathbf{)}\right\Vert ^{2}/\tau\right)  ,
\label{RF_Att_145}%
\end{equation}%
\begin{equation}
C_{ik}=(1-\gamma)\cdot\text{\textrm{softmax}}\left(  \left(  \tilde{y}%
_{i}-\tilde{y}_{k}\right)  ^{2}/\kappa\right)  . \label{RF_Att_146}%
\end{equation}

Expression (\ref{RF_Att_144}) can be rewritten as
\begin{align}
&  \sum_{i=1}^{T}\sum_{k=1}^{T}\tilde{y}_{k}\left(  D_{si}C_{ik}%
+C_{ik}\epsilon w_{i}+D_{si}\gamma v_{k}+\epsilon\gamma w_{i}v_{k}\right)
\nonumber\\
&  =\sum_{i=1}^{T}\sum_{k=1}^{T}\tilde{y}_{k}D_{si}C_{ik}+\epsilon\sum
_{i=1}^{T}\sum_{k=1}^{T}\tilde{y}_{k}C_{ik}w_{i}+\gamma\sum_{i=1}^{T}%
D_{si}\sum_{k=1}^{T}\tilde{y}_{k}v_{k}+\epsilon\gamma\sum_{k=1}^{T}\tilde
{y}_{k}v_{k}\nonumber\\
&  =R_{s}+\sum_{i=1}^{T}H_{i}w_{i}+\sum_{k=1}^{T}G_{sk}v_{k}.
\label{RF_Att_148}%
\end{align}
where
\begin{equation}
R_{s}=\sum_{i=1}^{T}\sum_{k=1}^{T}\tilde{y}_{k}D_{si}C_{ik},\ G_{sk}%
=\gamma\left(  \sum_{i=1}^{T}D_{si}+\epsilon\right)  \tilde{y}_{k}%
,\ H_{i}=\epsilon\sum_{k=1}^{T}\tilde{y}_{k}C_{ik}.
\end{equation}

Notations $D_{sk}$, $C_{ik}$, $R_{s}$, $G_{sk}$, $H_{i}$ do not depend on
$\mathbf{w}$ and $\mathbf{v}$ and are introduced for short.

It follows from (\ref{RF_Att_148}) that the optimization problem
(\ref{RF_Att_50}) is represented as
\begin{align}
&  \min_{\mathbf{w},\ \mathbf{v}}\sum_{s=1}^{n}L\left(  \tilde{y}_{s}%
,y_{s},\mathbf{w},\mathbf{v}\right) \nonumber\\
&  =\sum_{s=1}^{n}\left(  y_{s}-R_{s}-\sum_{i=1}^{T}H_{i}w_{i}-\sum_{k=1}%
^{T}G_{sk}v_{k}\right)  ^{2}, \label{RF_Att_150}%
\end{align}
subject to $w_{k}\geq0$, $v_{k}\geq0$, $k=1,...,T$, and $\sum_{k=1}^{T}%
w_{k}=1$, $\sum_{k=1}^{T}v_{k}=1$.

One of the advantages of the proposed SAT-RF is that it is simple from the
computational point of view because problem (\ref{RF_Att_150}) is the standard
quadratic programming problem which can be simply solved. Moreover, it has a
unique solution.

The optimal trainable parameters $\mathbf{w}$ and $\mathbf{v}$ can be also
computed by solving the linear optimization problem if to use the $L_{1}$-norm
for defining the loss function $L\left(  \tilde{y}_{s},y_{s},\mathbf{w}%
\right)  $. In this case, we replace (\ref{RF_Att_150}) with the following
objective function:
\begin{align}
&  \min_{\mathbf{w},\ \mathbf{v}}\sum_{s=1}^{n}L\left(  \tilde{y}_{s}%
,y_{s},\mathbf{w},\mathbf{v}\right) \nonumber\\
&  =\sum_{s=1}^{n}\left\vert y_{s}-R_{s}-\sum_{i=1}^{T}H_{i}w_{i}-\sum
_{k=1}^{T}G_{sk}v_{k}\right\vert ,
\end{align}

Denote%
\begin{equation}
Q_{s}=y_{s}-R_{s}-\sum_{i=1}^{T}H_{i}w_{i}-\sum_{k=1}^{T}G_{sk}v_{k}\mathbf{.}%
\end{equation}

Then we can write the following linear optimization problem with variables
$Q_{1},...,Q_{T}$, $\mathbf{w}$ and $\mathbf{v}$:%

\begin{equation}
\min_{\mathbf{w},\ \mathbf{v}}\sum_{s=1}^{n}Q_{s}, \label{RF_Att_34}%
\end{equation}
subject to $w_{k}\geq0$, $v_{k}\geq0$, $k=1,...,T$, and $\sum_{k=1}^{T}%
w_{k}=1$, $\sum_{k=1}^{T}v_{k}=1$, and
\begin{equation}
Q_{s}+\sum_{i=1}^{T}H_{i}w_{i}+\sum_{k=1}^{T}G_{sk}v_{k}\geq y_{s}%
-R_{s},\ s=1,...,n, \label{RF_Att_35}%
\end{equation}%
\begin{equation}
Q_{s}-\sum_{i=1}^{T}H_{i}w_{i}-\sum_{k=1}^{T}G_{sk}v_{k}\geq-y_{s}%
+R_{s},\ s=1,...,n. \label{RF_Att_36}%
\end{equation}

The above linear optimization problem has $3T$ variables and $2T+3n+2$ constraints.

\section{Multi-head self-attention}

One of the possible extensions of the self-attention mechanism is the
multi-head self-attention which is widely used to combine knowledge of the
self-attention via different representation of its tuning parameters. It turns
out that the multi-head self-attention can be incorporated into the
attention-based RF such that its trainable parameters are computed jointly with
parameters of the attention-based RF by solving a single quadratic
optimization problem.

Let us return to (\ref{RF_Att_123}) and rewrite the expression for estimating
$\tilde{y}$ as follows:
\begin{equation}
\tilde{y}=\sum_{i=1}^{T}\sum_{k(1)=1}^{T}\alpha\left(  \mathbf{x}%
,\mathbf{A}_{i}(\mathbf{x)},w_{i}\right)  \cdot\beta_{1}\left(  \tilde{y}%
_{i},\tilde{y}_{k(1)},v_{k(1)}^{(1)}\right)  \cdot\tilde{y}_{k(1)},
\label{RF_Att_160}%
\end{equation}
where $\mathbf{v}^{(1)}=(v_{1}^{(1)},...,v_{T}^{(1)})$ is the vector of
trainable variable of the first self-attention; $k(1)$ is the index
corresponding to the first self-attention.

Note that $\tilde{y}_{k(1)}$ as the value in terms of the attention mechanism
can be represented by means of the self-attention (\ref{RF_Att_121}). Hence,
(\ref{RF_Att_160}) can be rewritten as%
\begin{align}
\tilde{y}  &  =\sum_{i=1}^{T}\sum_{k(1)=1}^{T}\sum_{k(2)=1}^{T}\alpha\left(
\mathbf{x},\mathbf{A}_{i}(\mathbf{x)},w_{i}\right)  \cdot\beta_{1}\left(
\tilde{y}_{i},\tilde{y}_{k(1)},v_{k(1)}^{(1)}\right) \nonumber\\
&  \times\beta_{2}\left(  \tilde{y}_{i},\tilde{y}_{k(2)},v_{k(2)}%
^{(2)}\right)  \cdot\tilde{y}_{k(2)},
\end{align}
where $\mathbf{v}^{(2)}=(v_{1}^{(2)},...,v_{T}^{(2)})$ is the vector of
trainable variable of the second self-attention.

In the same way, we can continue writing self-attention operations and get
\begin{align}
\tilde{y}  &  =\sum_{i=1}^{T}\sum_{k(1)=1}^{T}\sum_{k(2)=1}^{T}\cdot\cdot
\cdot\sum_{k(t)=1}^{T}\alpha\left(  \mathbf{x},\mathbf{A}_{i}(\mathbf{x)}%
,w_{i}\right)  \cdot\beta_{1}\left(  \tilde{y}_{i},\tilde{y}_{k(1)}%
,v_{k(1)}^{(1)}\right) \nonumber\\
&  \times\beta_{2}\left(  \tilde{y}_{k(1)},\tilde{y}_{k(2)},v_{k(2)}%
^{(2)}\right)  \cdot\cdot\cdot\beta_{t}\left(  \tilde{y}_{k(t-1)},\tilde
{y}_{k(t)},v_{k(t)}^{(t)}\right)  \cdot\tilde{y}_{k(t)}, \label{RF_Att_170}%
\end{align}

In sum, we get a regression with $t$ self-attention operations having $t$
self-attention weights $\beta_{1},...,\beta_{t}$ with $t$ vectors of trainable
parameters $\mathbf{v}^{(1)},...,\mathbf{v}^{(t)}$ and the parameters
$\mathbf{w}$ of the attention.

Let us consider the case when the $j$-th self-attention weight $\beta
_{j}\left(  \tilde{y}_{k(j-1)},\tilde{y}_{k(j)},v_{k(j)}^{(j)}\right)  $ is
represented by the Huber's $\gamma_{j}$-contamination model as
\begin{equation}
\beta_{j}\left(  \tilde{y}_{k(j-1)},\tilde{y}_{k(j)},v_{k(j)}^{(j)}\right)
=(1-\gamma_{j})\cdot\text{\textrm{softmax}}\left(  \left(  \tilde{y}%
_{k(j-1)}-\tilde{y}_{k(j)}\right)  ^{2}/\kappa_{j}\right)  +\gamma_{j}\cdot
v_{k(j)}^{(j)}. \label{RF_Att_172}%
\end{equation}

Here $\gamma_{j}$ and $\kappa_{j}$ are tuning parameters of the $j$-th
contamination model. If parameters $\gamma_{j}$ and $\kappa_{j}$ are
differently defined for different $j=1,...,t$, then the obtained scheme can be
regarded as an analogue of the original multi-head self-attention. The random
choice of values of $\gamma_{j}$ and $\kappa_{j}$ is similar to the random
choice of initial weights in the neural network implementation of the
multi-head self-attention.

\begin{proposition}
\label{pr:Self_Att}If the self-attention weights $\alpha\left(  \mathbf{x}%
,\mathbf{A}_{i}(\mathbf{x)},w_{i}\right)  $ and $\beta_{j}\left(  \tilde
{y}_{k(j-1)},\tilde{y}_{k(j)},v_{k(j)}^{(j)}\right)  $ for all $j=1,...,t$,
are defined by (\ref{RF_Att_41}) and (\ref{RF_Att_172}), then $\tilde{y}$ in
(\ref{RF_Att_170}) is a linear function of parameters $\mathbf{w}$,
$\mathbf{v}^{(1)},...,\mathbf{v}^{(t)}$.
\end{proposition}

\begin{proof}
Introduce the following notations for short:%
\begin{equation}
C_{k(j)}=(1-\gamma_{j})\cdot\text{\textrm{softmax}}\left(  \left(  \tilde
{y}_{k(j-1)}-\tilde{y}_{k(j)}\right)  ^{2}/\kappa_{j}\right)  .
\end{equation}
Then we write
\begin{align}
\tilde{y}=  &  \sum_{i=1}^{T}\left(  D_{si}+\epsilon w_{i}\right)
\sum_{k(1)=1}^{T}\left(  C_{k(1)}+\gamma_{1}v_{k(1)}^{(1)}\right) \nonumber\\
&  \cdot\cdot\cdot\sum_{k(t-1)=1}^{T}\left(  C_{k(t-1)}+\gamma_{t-1}%
v_{k(t-1)}^{(t-1)}\right)  \sum_{k(t)=1}^{T}\left(  C_{k(t)}+\gamma
_{t}v_{k(t)}^{(t)}\right)  \tilde{y}_{k(t)}. \label{RF_Att_175}%
\end{align}
Let us consider $\beta_{t-1}$ and $\beta_{t}$
\begin{align}
&  \sum_{k(t-1)=1}^{T}\left(  C_{k(t-1)}+\gamma_{t-1}v_{k(t-1)}^{(t-1)}%
\right)  \sum_{k(t)=1}^{T}\left(  C_{k(t)}+\gamma_{t}v_{k(t)}^{(t)}\right)
\tilde{y}_{k(t)}\nonumber\\
&  =\sum_{k(t-1)=1}^{T}\sum_{k(t)=1}^{T}\left(  C_{k(t-1)}+\gamma
_{t-1}v_{k(t-1)}^{(t-1)}\right)  \left(  C_{k(t)}+\gamma_{t}v_{k(t)}%
^{(t)}\right)  \tilde{y}_{k(t)}\nonumber\\
&  =\sum_{k(t)=1}^{T}\left(  \sum_{k(t-1)=1}^{T}C_{k(t-1)}C_{k(t)}%
+\gamma_{t-1}\sum_{k(t-1)=1}^{T}C_{k(t)}v_{k(t-1)}^{(t-1)}\right. \nonumber\\
&  \left.  +\gamma_{t}v_{k(t)}^{(t)}\sum_{k(t-1)=1}^{T}C_{k(t-1)}+\gamma
_{t-1}\gamma_{t}v_{k(t)}^{(t)}\sum_{k(t-1)=1}^{T}v_{k(t-1)}^{(t-1)}\right)
\tilde{y}_{k(t)}.
\end{align}
It should be noted that $\sum_{k(t-1)=1}^{T}v_{k(t-1)}^{(t-1)}=1$. This
implies that the product of $\beta_{t-1}$ and $\beta_{t}$ linearly depends on
$v_{k(t-1)}^{(t-1)}$ and $v_{k(t)}^{(t)}$. Let us rewrite the obtained product
as follows:
\begin{align}
&  \sum_{k(t-1)=1}^{T}\left(  \sum_{k(t)=1}^{T}\tilde{y}_{k(t)}C_{k(t-1)}%
C_{k(t)}+\gamma_{t-1}\sum_{k(t)=1}^{T}\tilde{y}_{k(t)}C_{k(t)}v_{k(t-1)}%
^{(t-1)}\right. \nonumber\\
&  \left.  +\gamma_{t}\sum_{k(t)=1}^{T}\tilde{y}_{k(t)}v_{k(t)}^{(t)}%
C_{k(t-1)}+\gamma_{t-1}\gamma_{t}\sum_{k(t)=1}^{T}\tilde{y}_{k(t)}%
v_{k(t)}^{(t)}\right) \\
&  =\sum_{k(t-1)=1}^{T}\left(  G_{k(t)}+r_{k(t-1)}^{(t-1)}\right)  .
\end{align}
where
\begin{equation}
G_{k(t)}=\sum_{k(t)=1}^{T}\tilde{y}_{k(t)}C_{k(t-1)}C_{k(t)},
\end{equation}
is a constant, and
\begin{align}
r_{k(t-1)}^{(t-1)}  &  =v_{k(t-1)}^{(t-1)}\gamma_{t-1}\sum_{k(t)=1}^{T}%
\tilde{y}_{k(t)}C_{k(t)}\nonumber\\
&  +\gamma_{t}\sum_{k(t)=1}^{T}\tilde{y}_{k(t)}v_{k(t)}^{(t)}C_{k(t-1)}%
+\gamma_{t-1}\gamma_{t}\sum_{k(t)=1}^{T}\tilde{y}_{k(t)}v_{k(t)}^{(t)},
\end{align}
is a new variable instead of $v_{k(t-1)}^{(t-1)}$ and $v_{k(t)}^{(t)}$.

In the same way, we can consider all other multipliers in (\ref{RF_Att_175})
starting from pair $\beta_{t-2}$ and $\beta_{t-1}$, as was to be proved.
\end{proof}

The linearity of $\tilde{y}$ as the function of variables $\mathbf{w}$,
$\mathbf{v}^{(1)},...,\mathbf{v}^{(t)}$ implies that the optimization problem
for training these variables can be reduced to the quadratic or linear
optimization problems like (\ref{RF_Att_150}) or (\ref{RF_Att_34}). However,
writing the optimization problem by $t>2$ is a hard problem. Therefore, the
multi-head self-attention was given to show the fundamental possibility of
generalizing the self-attention-based RF. An approach for efficient
representation of the multi-head attention can be regarded as a separate
problem whose solution is a direction for further research.

\section{Numerical experiments}

In order to study the proposed approach for solving regression problems, we
apply datasets which are taken from open sources: the dataset Diabetes is
available in the corresponding R Packages; datasets Friedman 1, 2 3 can be
found at site: https://www.stat.berkeley.edu/\symbol{126}breiman/bagging.pdf;
Regression and Sparse datasets are available in package \textquotedblleft
Scikit-Learn\textquotedblright. The proposed algorithm is evaluated and
investigated also by the following publicly available datasets from the UCI
Machine Learning Repository \cite{Dua:2019}: Wine Red, Boston Housing,
Concrete, Yacht Hydrodynamics, Airfoil. A brief introduction about these data
sets are given in Table \ref{t:regres_datasets} where $m$ and $n$ are numbers
of features and examples, respectively. A more detailed information can be
found from the aforementioned data resources.%

\begin{table}[tbp] \centering
\caption{A brief introduction about the regression data sets}%
\begin{tabular}
[c]{cccc}\hline
Data set & Abbreviation & $m$ & $n$\\\hline
Diabetes & Diabetes & $10$ & $442$\\\hline
Friedman 1 & Friedman 1 & $10$ & $100$\\\hline
Friedman 2 & Friedman 2 & $4$ & $100$\\\hline
Friedman 3 & Friedman 3 & $4$ & $100$\\\hline
Scikit-Learn Regression & Regression & $100$ & $100$\\\hline
Scikit-Learn Sparse Uncorrelated & Sparse & $10$ & $100$\\\hline
UCI Wine red & Wine & $11$ & $1599$\\\hline
UCI Boston Housing & Boston & $13$ & $506$\\\hline
UCI Concrete & Concrete & $8$ & $1030$\\\hline
UCI Yacht Hydrodynamics & Yacht & $6$ & $308$\\\hline
UCI Airfoil & Airfoil & $5$ & $1503$\\\hline
\end{tabular}
\label{t:regres_datasets}%
\end{table}%

The coefficient of determination denoted $R^{2}$ and the mean absolute error
(MAE) are used for the regression evaluation. The greater the value of the
coefficient of determination and the smaller the MAE, the better results we
get. In all tables, we compare $R^{2}$ and the MAE for three cases:

\begin{enumerate}
\item \textbf{RF},\textbf{ ERT}: the original RF or the ERT without the
softmax and without attention model;

\item \textbf{Softmax} model: the RF or the ERT with softmax operations
without trainable parameters, i.e., weights of trees are determined under
conditions $\epsilon=0$ and $\gamma=0$.

\item \textbf{SAT-RF-y}, \textbf{SAT-RF-x}, \textbf{SAT-RF-yx}: the
corresponding modifications of the SAT-RF models.
\end{enumerate}

The best results in all tables are shown in bold. Moreover, the optimal values
of the contamination parameters $\epsilon_{opt}$ and $\gamma_{opt}$ are
provided. The case $\epsilon_{opt}=1$ means that weights of trees are totally
determined by the tree results and do not depend on each example. This case
coincides with the weighted RF proposed in \cite{Utkin-Konstantinov-etal-20}.
The case $\epsilon_{opt}=0$ means that weights of trees are determined only by
the softmax function (with or without trainable parameters). Cases
$\gamma_{opt}=1$ and $\gamma_{opt}=0$ have the same sense.

Decision trees in numerical experiments are trained such that at least $10$
examples fall into every leaf of trees. This condition is used to get
desirable estimates of vectors $\mathbf{A}_{k}(\mathbf{x}_{s}\mathbf{)}$.

Every RF or ERT consists of $100$ decision trees. In order to optimize the
model parameters $\epsilon$ and $\tau$ in numerical experiments, we perform a
3-fold cross-validation on the training set which consists of $n_{\text{tr}%
}=4n/5$ instances. The cross-validation is performed with $100$ repetitions.
This procedure is realized by considering all possible values of $\epsilon$
and $\tau$ in a predefined grid. The testing set consisting of $n_{\text{test}%
}=n/5$ instances is used for computing the accuracy measures of the whole model.

First, we consider SAT-RF-y. It has four tuning parameters $\epsilon$,
$\gamma$, $\kappa$ and $\tau$, which may significantly impact on predictions.
Therefore, the best predictions are calculated at a predefined grid of the
parameters $\epsilon$, $\gamma$, and a cross-validation procedure is
subsequently used to select an appropriate values of $\epsilon$ and $\gamma$.
Moreover, parameters $\kappa$ and $\tau$ are taken equal to $1$. If SAT-RF-y
provides outperforming results with fixed $\kappa$ and $\tau$, then optimal
choice of these parameters will improve the model.

Measures $R^{2}$ and MAE for three models (RF, Softmax and SAT-RF-y) are shown
in Table \ref{t:regression_1_10}. The results are obtained by training the RF
and the parameter vectors $\mathbf{w}$ and $\mathbf{v}$ on the regression
datasets. It can be seen from Table \ref{t:regression_1_10} that SAT-RF-y
outperforms the RF and the Softmax models almost for all datasets. The same
results are shown in Table \ref{t:regression_1_11} under condition that the RF
in experiments is replaced with the ERT. One can again see from Table
\ref{t:regression_1_11} that SAT-RF-y outperforms the ERT and the Softmax
models for most datasets. Optimal values of tuning parameters $\epsilon$ and
$\gamma$ are also shown in Tables \ref{t:regression_1_10} and
\ref{t:regression_1_11}. It is also interesting to point out that SAT-RF-y and
Softmax using the RF provide the same measures $R^{2}$ or MAE for datasets
Diabetes and Sparse. Indeed, it can be seen from Table \ref{t:regression_1_10}
that $\epsilon_{opt}=\gamma_{opt}=0$. This implies that only softmax
operations without trainable parameters define the attention weights. It can
be seen from Table \ref{t:regression_1_11} that this case does not take place
for SAT-RF-y using ERT.%

\begin{table}[tbp] \centering
\caption{Measures $R^2$ and MAE for comparison of models (the RF, the Softmax model, SAT-RF-y) trained on regression datasets}%
\begin{tabular}
[c]{ccccccccc}\hline
&  &  & \multicolumn{3}{c}{$R^{2}$} & \multicolumn{3}{c}{MAE}\\\hline
Data set & $\epsilon_{opt}$ & $\gamma_{opt}$ & RF & Softmax & SAT-RF-y & RF &
Softmax & SAT-RF-y\\\hline
Diabetes & $0$ & $0$ & $0.416$ & $\mathbf{0.424}$ & $\mathbf{0.424}$ & $44.92$
& $\mathbf{44.66}$ & $\mathbf{44.66}$\\\hline
Friedman 1 & $0$ & $1$ & $0.459$ & $0.438$ & $\mathbf{0.470}$ &
$\mathbf{2.540}$ & $2.589$ & $\mathbf{2.540}$\\\hline
Friedman 2 & $0.5$ & $0$ & $0.841$ & $0.849$ & $\mathbf{0.878}$ & $111.7$ &
$109.5$ & $\mathbf{100.9}$\\\hline
Friedman 3 & $0$ & $1$ & $0.625$ & $0.625$ & $\mathbf{0.682}$ & $0.155$ &
$0.156$ & $\mathbf{0.133}$\\\hline
Regression & $0.5$ & $0.25$ & $0.380$ & $0.367$ & $\mathbf{0.454}$ & $109.1$ &
$110.2$ & $\mathbf{100.5}$\\\hline
Sparse & $0$ & $0$ & $0.470$ & $\mathbf{0.522}$ & $\mathbf{0.522}$ & $1.908$ &
$\mathbf{1.802}$ & $\mathbf{1.802}$\\\hline
Airfoil & $1$ & $0.75$ & $0.823$ & $0.820$ & $\mathbf{0.843}$ & $2.203$ &
$2.231$ & $\mathbf{2.069}$\\\hline
Boston & $0$ & $0.5$ & $0.814$ & $0.818$ & $\mathbf{0.823}$ & $2.539$ &
$2.508$ & $\mathbf{2.494}$\\\hline
Concrete & $0$ & $1$ & $0.845$ & $0.841$ & $\mathbf{0.857}$ & $4.855$ &
$4.948$ & $\mathbf{4.694}$\\\hline
Wine & $0.25$ & $0.25$ & $\mathbf{0.433}$ & $0.423$ & $0.424$ &
$\mathbf{0.451}$ & $0.460$ & $0.459$\\\hline
Yacht & $1$ & $0.5$ & $0.981$ & $0.981$ & $\mathbf{0.989}$ & $1.004$ & $1.006$
& $\mathbf{0.787}$\\\hline
\end{tabular}
\label{t:regression_1_10}%
\end{table}%
%

\begin{table}[tbp] \centering
\caption{Measures $R^2$ and MAE for comparison of models (the ERT, the Softmax model, SAT-RF-y) trained on regression datasets}%
\begin{tabular}
[c]{ccccccccc}\hline
&  &  & \multicolumn{3}{c}{$R^{2}$} & \multicolumn{3}{c}{MAE}\\\hline
Data set & $\epsilon_{opt}$ & $\gamma_{opt}$ & ERT & Softmax & SAT-RF-y &
ERT & Softmax & SAT-RF-y\\\hline
Diabetes & $0.75$ & $0$ & $0.456$ & $\mathbf{0.458}$ & $0.453$ & $44.50$ &
$\mathbf{44.38}$ & $44.51$\\\hline
Friedman 1 & $1$ & $0.75$ & $0.471$ & $0.471$ & $\mathbf{0.521}$ & $2.502$ &
$2.502$ & $\mathbf{2.414}$\\\hline
Friedman 2 & $1$ & $0.25$ & $0.813$ & $0.813$ & $\mathbf{0.939}$ & $123.03$ &
$122.66$ & $\mathbf{73.77}$\\\hline
Friedman 3 & $0$ & $1$ & $0.570$ & $0.570$ & $\mathbf{0.739}$ & $0.179$ &
$0.179$ & $\mathbf{0.138}$\\\hline
Regression & $1$ & $0$ & $0.402$ & $0.403$ & $\mathbf{0.455}$ & $106.3$ &
$106.2$ & $\mathbf{101.8}$\\\hline
Sparse & $0$ & $0.25$ & $0.452$ & $0.514$ & $\mathbf{0.531}$ & $1.994$ &
$1.870$ & $\mathbf{1.830}$\\\hline
Airfoil & $1$ & $0.75$ & $0.802$ & $0.802$ & $\mathbf{0.837}$ & $2.370$ &
$2.370$ & $\mathbf{2.127}$\\\hline
Boston & $0.5$ & $0.75$ & $0.831$ & $0.833$ & $\mathbf{0.837}$ & $2.481$ &
$2.467$ & $\mathbf{2.453}$\\\hline
Concrete & $0$ & $1$ & $0.851$ & $0.851$ & $\mathbf{0.869}$ & $4.892$ &
$4.892$ & $\mathbf{4.650}$\\\hline
Wine & $1$ & $0.25$ & $0.418$ & $0.418$ & $\mathbf{0.419}$ & $0.464$ &
$\mathbf{0.463}$ & $\mathbf{0.463}$\\\hline
Yacht & $0$ & $1$ & $\mathbf{0.988}$ & $\mathbf{0.988}$ & $\mathbf{0.988}$ &
$0.824$ & $0.824$ & $\mathbf{0.818}$\\\hline
\end{tabular}
\label{t:regression_1_11}%
\end{table}%

The next modification for studying is SAT-RF-x, The corresponding results of
numerical experiments under the same condition as experiments with SAT-RF-y
are shown in Tables \ref{t:regression_1_14}-\ref{t:regression_1_15}. However,
if to compare these results with results given in Tables
\ref{t:regression_1_10}-\ref{t:regression_1_11}, then they are mainly inferior
to SAT-RF-y and comparable to this modification when the RF is used. The same
can be seen from Table \ref{t:regression_compar} where SAT-RF-x is compared
with SAT-RF-y and SAT-RF-yx. The results can be explained as follows. A large
distance between $A_{i}(\mathbf{x})$ and $A_{j}(\mathbf{x})$ mainly says about
a large difference between subsets of examples used for training the $i$-th
and the $j$-th trees. However, this distance does not say about predictions of
trees which are transformed by using the self-attention.%

\begin{table}[tbp] \centering
\caption{Measures $R^2$ and MAE for comparison of models (the RF, the Softmax model, SAT-RF-x) trained on regression datasets}%
\begin{tabular}
[c]{ccccccccc}\hline
&  &  & \multicolumn{3}{c}{$R^{2}$} & \multicolumn{3}{c}{MAE}\\\hline
Data set & $\epsilon_{opt}$ & $\gamma_{opt}$ & RF & Softmax & SAT-RF-x & RF &
Softmax & SAT-RF-x\\\hline
Diabetes & $0$ & $0$ & $0.405$ & $\mathbf{0.416}$ & $\mathbf{0.416}$ & $44.92$
& $\mathbf{44.87}$ & $\mathbf{44.87}$\\\hline
Friedman 1 & $0$ & $1$ & $0.459$ & $0.438$ & $\mathbf{0.470}$ & $2.540$ &
$2.589$ & $\mathbf{2.540}$\\\hline
Friedman 2 & $1$ & $1$ & $0.841$ & $0.834$ & $\mathbf{0.872}$ & $111.7$ &
$114.5$ & $\mathbf{103.7}$\\\hline
Friedman 3 & $0.5$ & $0.5$ & $0.625$ & $0.623$ & $\mathbf{0.684}$ & $0.154$ &
$0.156$ & $\mathbf{0.134}$\\\hline
Regression & $0.75$ & $0$ & $0.380$ & $0.374$ & $\mathbf{0.451}$ & $109.1$ &
$110.0$ & $\mathbf{100.4}$\\\hline
Sparse & $0$ & $0$ & $0.470$ & $\mathbf{0.488}$ & $\mathbf{0.488}$ & $1.908$ &
$\mathbf{1.860}$ & $\mathbf{1.860}$\\\hline
Airfoil & $0.25$ & $1$ & $0.823$ & $0.820$ & $\mathbf{0.843}$ & $2.203$ &
$2.231$ & $\mathbf{2.070}$\\\hline
Boston & $0.25$ & $0.5$ & $0.814$ & $0.814$ & $\mathbf{0.821}$ & $2.539$ &
$2.539$ & $\mathbf{2.518}$\\\hline
Concrete & $0$ & $1$ & $0.845$ & $0.841$ & $\mathbf{0.857}$ & $4.855$ &
$4.948$ & $\mathbf{4.694}$\\\hline
Wine & $0$ & $0.75$ & $0.433$ & $0.421$ & $\mathbf{0.422}$ & $0.451$ & $0.461$
& $\mathbf{0.459}$\\\hline
Yacht & $0$ & $1$ & $0.981$ & $0.981$ & $\mathbf{0.989}$ & $1.004$ & $1.004$ &
$\mathbf{0.787}$\\\hline
\end{tabular}
\label{t:regression_1_14}%
\end{table}%
%

\begin{table}[tbp] \centering
\caption{Measures $R^2$ and MAE for comparison of models (the ERT, the Softmax model, SAT-RF-x) trained on regression datasets}%
\begin{tabular}
[c]{ccccccccc}\hline
&  &  & \multicolumn{3}{c}{$R^{2}$} & \multicolumn{3}{c}{MAE}\\\hline
Data set & $\epsilon_{opt}$ & $\gamma_{opt}$ & ERT & Softmax & SAT-RF-x &
ERT & Softmax & SAT-RF-x\\\hline
Diabetes & $0$ & $0.75$ & $\mathbf{0.449}$ & $\mathbf{0.449}$ & $0.442$ &
$44.41$ & $\mathbf{44.37}$ & $44.57$\\\hline
Friedman 1 & $0$ & $1$ & $0.471$ & $0.471$ & $\mathbf{0.513}$ & $2.502$ &
$2.50$ & $\mathbf{2.426}$\\\hline
Friedman 2 & $0$ & $1$ & $0.813$ & $0.813$ & $\mathbf{0.930}$ & $123.0$ &
$123.0$ & $\mathbf{74.50}$\\\hline
Friedman 3 & $0$ & $1$ & $0.570$ & $0.570$ & $\mathbf{0.739}$ & $0.179$ &
$0.179$ & $\mathbf{0.138}$\\\hline
Regression & $1$ & $0$ & $0.402$ & $0.443$ & $\mathbf{0.493}$ & $106.3$ &
$102.5$ & $\mathbf{95.95}$\\\hline
Sparse & $0$ & $0.25$ & $0.452$ & $0.501$ & $\mathbf{0.518}$ & $1.994$ &
$1.887$ & $\mathbf{1.851}$\\\hline
Airfoil & $0.5$ & $1$ & $0.802$ & $0.802$ & $\mathbf{0.837}$ & $2.370$ &
$2.370$ & $\mathbf{2.128}$\\\hline
Boston & $1$ & $0.25$ & $0.831$ & $0.835$ & $\mathbf{0.843}$ & $2.481$ &
$2.447$ & $\mathbf{2.402}$\\\hline
Concrete & $0$ & $1$ & $0.851$ & $0.851$ & $\mathbf{0.863}$ & $4.892$ &
$4.892$ & $\mathbf{4.650}$\\\hline
Wine & $1$ & $0$ & $\mathbf{0.418}$ & $0.417$ & $\mathbf{0.418}$ &
$\mathbf{0.462}$ & $0.463$ & $\mathbf{0.462}$\\\hline
Yacht & $0$ & $1$ & $\mathbf{0.988}$ & $\mathbf{0.988}$ & $\mathbf{0.988}$ &
$0.824$ & $0.824$ & $\mathbf{0.818}$\\\hline
\end{tabular}
\label{t:regression_1_15}%
\end{table}%

Results of numerical experiments with SAT-RF-yx are presented in Tables
\ref{t:regression_1_12}-\ref{t:regression_1_13}. One can see from the tables
that SAT-RF-yx outperforms other models. In particular, it is shown in Table
\ref{t:regression_1_12} that SAT-RF-yx provides better results for all
datasets except for the Wine dataset. The same can be said about models
constructed by using ERTs. The corresponding results are shown in Table
\ref{t:regression_1_13}. If we compare results from Table
\ref{t:regression_1_12} with results from Table \ref{t:regression_1_13}, then
it is interesting to point out that the use of ERTs significantly improves the models.%

\begin{table}[tbp] \centering
\caption{Measures $R^2$ and MAE for comparison of models (the RF, the Softmax model, SAT-RF-yx) trained on regression datasets}%
\begin{tabular}
[c]{ccccccccc}\hline
&  &  & \multicolumn{3}{c}{$R^{2}$} & \multicolumn{3}{c}{MAE}\\\hline
Data set & $\epsilon_{opt}$ & $\gamma_{opt}$ & RF & Softmax & SAT-RF-yx & RF &
Softmax & SAT-RF-yx\\\hline
Diabetes & $0$ & $0$ & $0.416$ & $\mathbf{0.422}$ & $\mathbf{0.422}$ &
$\mathbf{44.92}$ & $45.01$ & $45.01$\\\hline
Friedman 1 & $1$ & $0.75$ & $0.459$ & $0.440$ & $\mathbf{0.489}$ & $2.540$ &
$2.574$ & $\mathbf{2.509}$\\\hline
Friedman 2 & $1$ & $0.5$ & $0.841$ & $0.788$ & $\mathbf{0.882}$ & $111.7$ &
$125.0$ & $\mathbf{95.84}$\\\hline
Friedman 3 & $0.5$ & $0.5$ & $0.625$ & $0.628$ & $\mathbf{0.685}$ & $0.154$ &
$0.155$ & $\mathbf{0.133}$\\\hline
Regression & $1$ & $0.25$ & $0.380$ & $0.363$ & $\mathbf{0.488}$ & $109.1$ &
$111.4$ & $\mathbf{96.50}$\\\hline
Sparse & $0$ & $0$ & $0.470$ & $0.531$ & $\mathbf{0.540}$ & $1.908$ & $1.783$
& $\mathbf{1.775}$\\\hline
Airfoil & $0.25$ & $1$ & $0.823$ & $0.820$ & $\mathbf{0.849}$ & $2.203$ &
$2.231$ & $\mathbf{2.070}$\\\hline
Boston & $0.25$ & $0.75$ & $0.814$ & $0.814$ & $\mathbf{0.824}$ & $2.539$ &
$2.546$ & $\mathbf{2.501}$\\\hline
Concrete & $1$ & $0.75$ & $0.845$ & $0.841$ & $\mathbf{0.866}$ & $4.834$ &
$4.921$ & $\mathbf{4.651}$\\\hline
Wine & $0.25$ & $0.5$ & $\mathbf{0.433}$ & $0.422$ & $0.429$ & $\mathbf{0.451}%
$ & $0.461$ & $0.458$\\\hline
Yacht & $1$ & $0.5$ & $0.981$ & $0.971$ & $\mathbf{0.989}$ & $1.004$ & $1.237$
& $\mathbf{0.790}$\\\hline
\end{tabular}
\label{t:regression_1_12}%
\end{table}%
%

\begin{table}[tbp] \centering
\caption{Measures $R^2$ and MAE for comparison of models (the ERT, the Softmax model, SAT-RF-yx) trained on regression datasets}%
\begin{tabular}
[c]{ccccccccc}\hline
&  &  & \multicolumn{3}{c}{$R^{2}$} & \multicolumn{3}{c}{MAE}\\\hline
Data set & $\epsilon_{opt}$ & $\gamma_{opt}$ & ERT & Softmax & SAT-RF-yx &
ERT & Softmax & SAT-RF-yx\\\hline
Diabetes & $0$ & $0$ & $0.438$ & $\mathbf{0.439}$ & $\mathbf{0.439}$ & $44.55$
& $\mathbf{44.26}$ & $\mathbf{44.26}$\\\hline
Friedman 1 & $0$ & $1$ & $0.471$ & $0.471$ & $\mathbf{0.513}$ & $2.502$ &
$2.502$ & $\mathbf{2.426}$\\\hline
Friedman 2 & $0$ & $1$ & $0.813$ & $0.813$ & $\mathbf{0.930}$ & $123.0$ &
$123.0$ & $\mathbf{74.49}$\\\hline
Friedman 3 & $1$ & $0.5$ & $0.570$ & $0.570$ & $\mathbf{0.751}$ & $0.179$ &
$0.179$ & $\mathbf{0.137}$\\\hline
Regression & $0$ & $0.75$ & $0.402$ & $0.411$ & $\mathbf{0.449}$ & $106.3$ &
$105.4$ & $\mathbf{101.1}$\\\hline
Sparse & $0$ & $0.25$ & $0.452$ & $0.518$ & $\mathbf{0.542}$ & $1.994$ &
$1.863$ & $\mathbf{1.822}$\\\hline
Airfoil & $0.5$ & $1$ & $0.802$ & $0.802$ & $\mathbf{0.841}$ & $2.370$ &
$2.370$ & $\mathbf{2.128}$\\\hline
Boston & $0.75$ & $0$ & $0.831$ & $0.836$ & $\mathbf{0.844}$ & $2.481$ &
$2.452$ & $\mathbf{2.427}$\\\hline
Concrete & $1$ & $0.5$ & $0.839$ & $0.839$ & $\mathbf{0.860}$ & $5.119$ &
$5.128$ & $\mathbf{4.689}$\\\hline
Wine & $0$ & $0.75$ & $0.418$ & $0.417$ & $\mathbf{0.447}$ & $0.464$ & $0.463$
& $\mathbf{0.462}$\\\hline
Yacht & $0$ & $1$ & $\mathbf{0.988}$ & $\mathbf{0.988}$ & $\mathbf{0.988}$ &
$0.824$ & $0.824$ & $\mathbf{0.818}$\\\hline
\end{tabular}
\label{t:regression_1_13}%
\end{table}%

Fig. \ref{f:Sparse} illustrates how measure $R^{2}$ depends on the attention
parameters $\tau$, $\epsilon$ (the left picture) and on the self-attention
parameters $\kappa$, $\gamma$ (the right picture) for the Sparse dataset. It
is interesting to see from Fig. \ref{f:Sparse} that $R^{2}$ achieves its
maximum by $\tau=1$ and $\epsilon=0$ or $\epsilon=0.25$. At the same time,
$R^{2}$ achieves its maximum by $\kappa<1$ and $\gamma=0.25$. The optimal
values $\epsilon$ and $\gamma$ coincide with the corresponding optimal values
shown in Table \ref{t:regression_1_13}. Figs. \ref{f:Friedman1},
\ref{f:boston}, \ref{f:wine} illustrate the same dependencies for the Friedman
1, Boston, Wine datasets, respectively.%

\begin{figure}
[ptb]
\begin{center}
\includegraphics[
height=2.4326in,
width=5.3313in
]%
{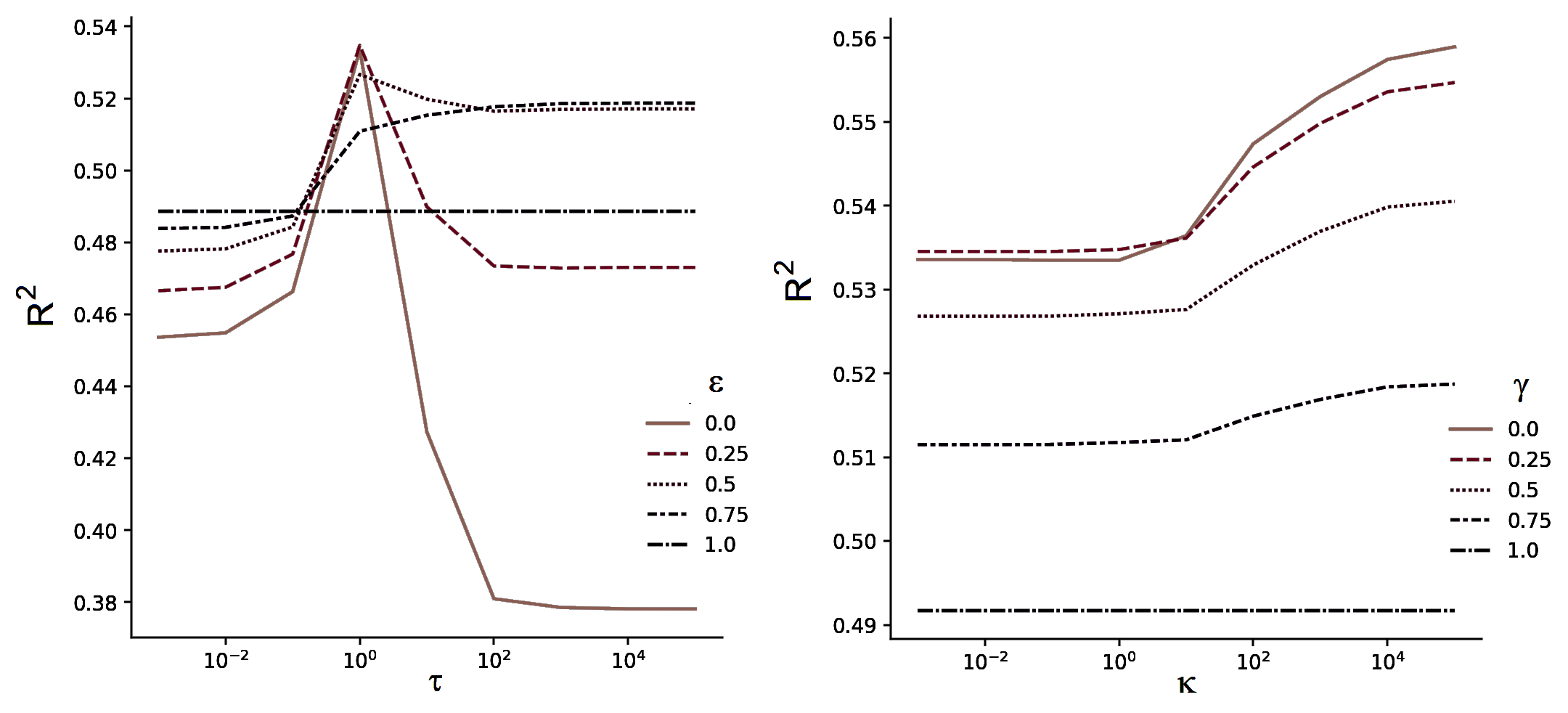}%
\caption{Measure $R^{2}$ as a function of the attention parameters $\tau$,
$\epsilon$ (left) and the self-attention parameters $\kappa$, $\gamma$ (right)
for the Sparse dataset and SAT-RF-yx using the ERT}%
\label{f:Sparse}%
\end{center}
\end{figure}
%

\begin{figure}
[ptb]
\begin{center}
\includegraphics[
height=2.4984in,
width=5.3653in
]%
{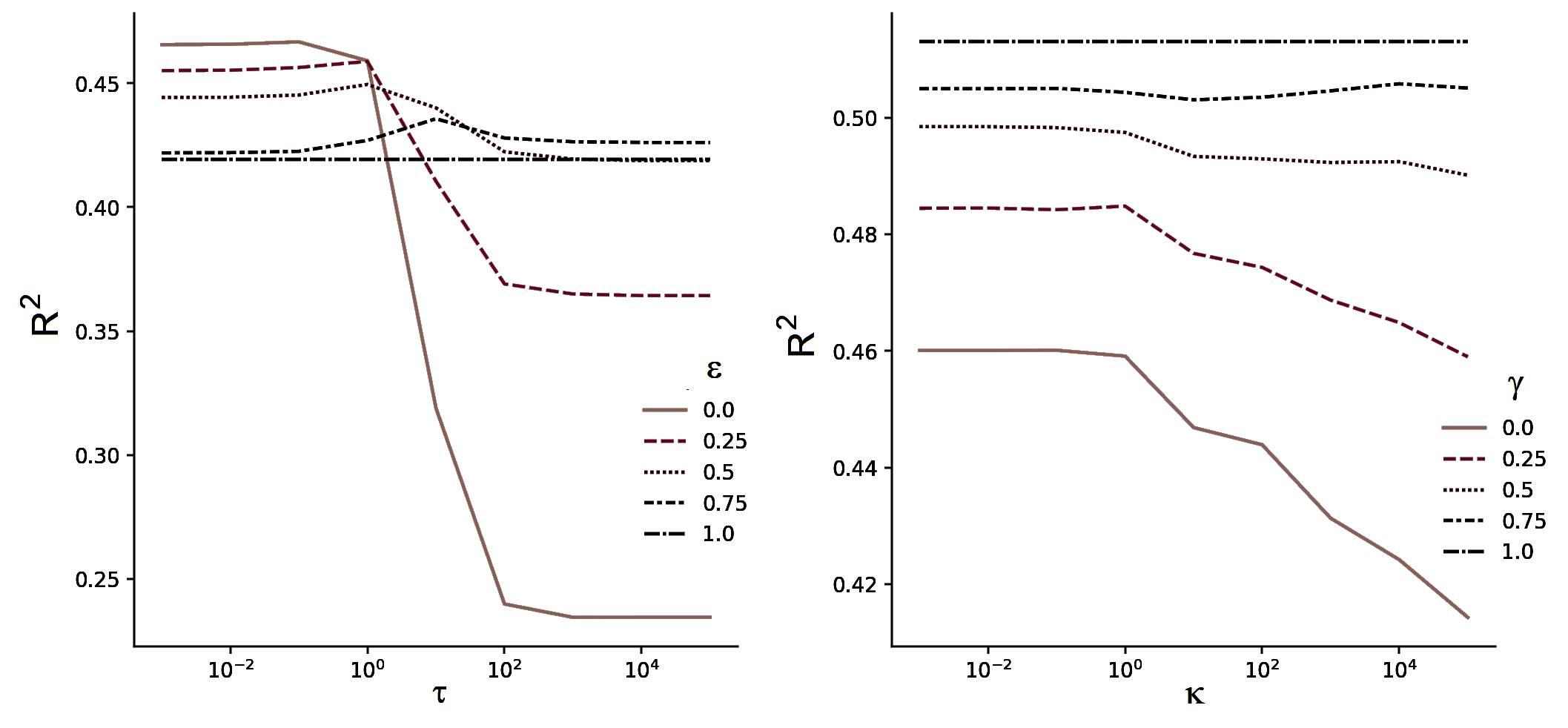}%
\caption{Measure $R^{2}$ as a function of the attention parameters $\tau$,
$\epsilon$ (left) and the self-attention parameters $\kappa$, $\gamma$ (right)
for the Fridman 1 dataset and SAT-RF-yx using the ERT}%
\label{f:Friedman1}%
\end{center}
\end{figure}
%

\begin{figure}
[ptb]
\begin{center}
\includegraphics[
height=2.5469in,
width=5.4114in
]%
{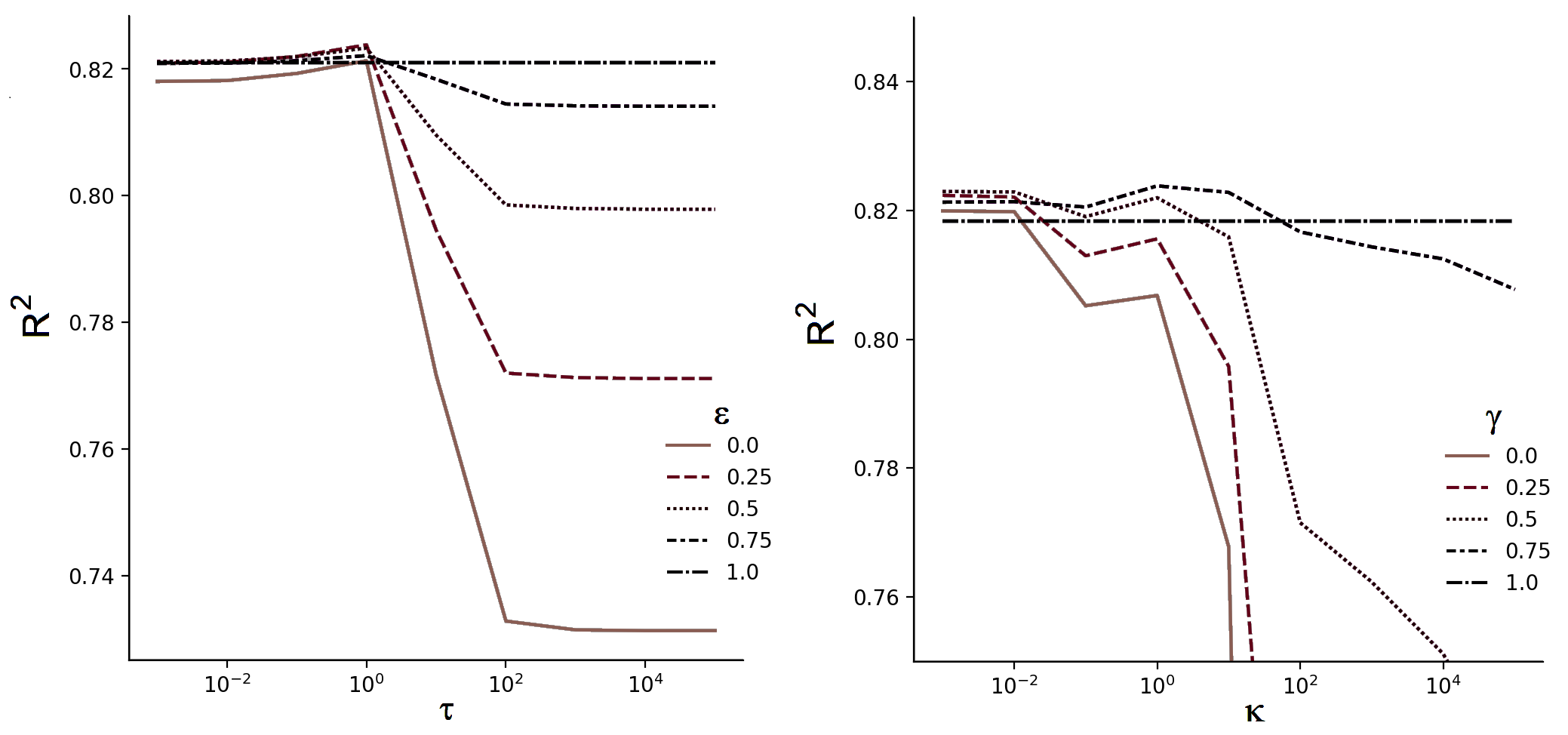}%
\caption{Measure $R^{2}$ as a function of the attention parameters $\tau$,
$\epsilon$ (left) and the self-attention parameters $\kappa$, $\gamma$ (right)
for the Boston dataset and SAT-RF-yx using the ERT}%
\label{f:boston}%
\end{center}
\end{figure}
%

\begin{figure}
[ptb]
\begin{center}
\includegraphics[
height=2.663in,
width=5.4501in
]%
{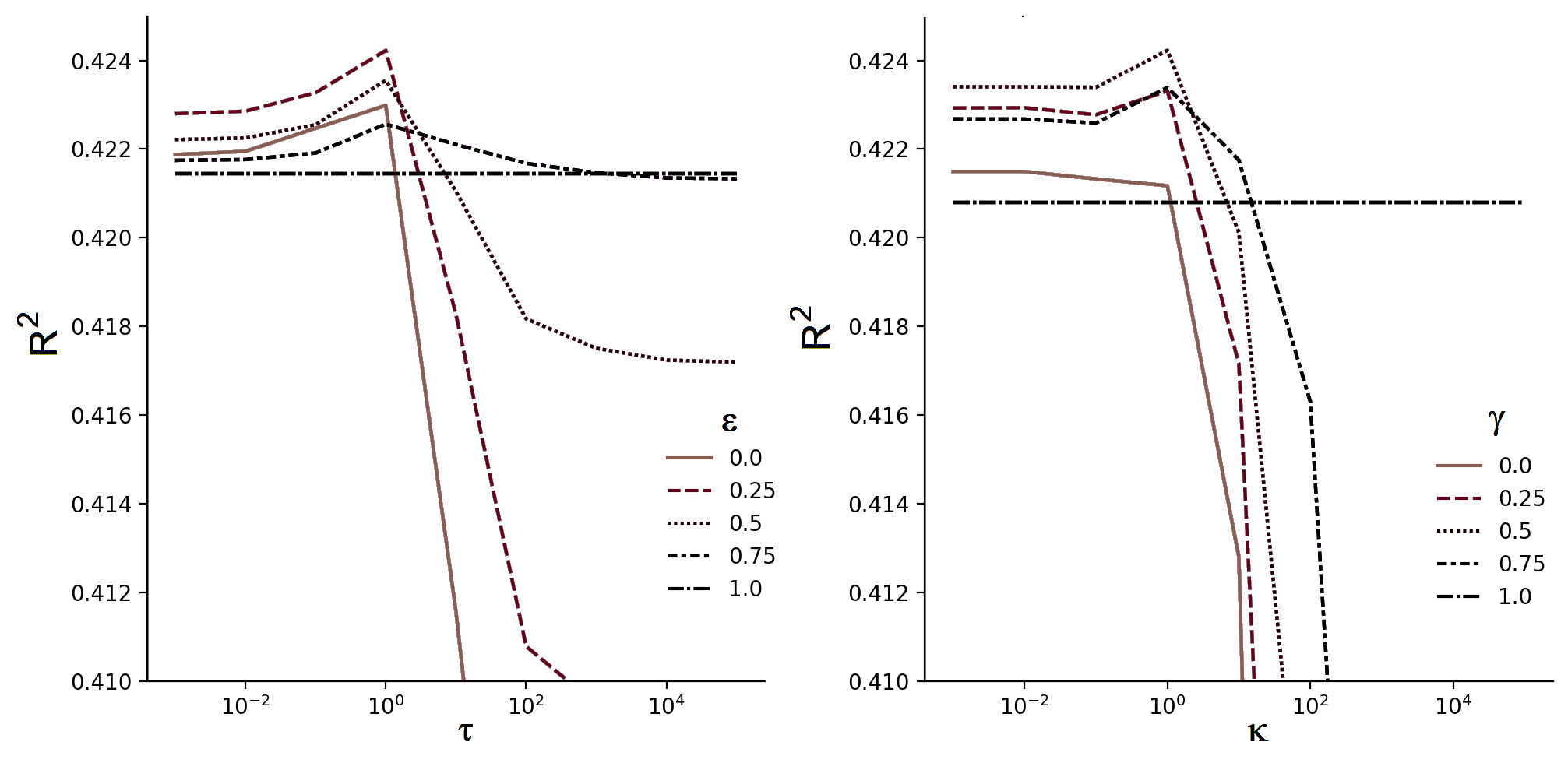}%
\caption{Measure $R^{2}$ as a function of the attention parameters $\tau$,
$\epsilon$ (left) and the self-attention parameters $\kappa$, $\gamma$ (right)
for the Wine dataset and SAT-RF-yx using the ERT}%
\label{f:wine}%
\end{center}
\end{figure}

The best results for all modifications are illustrated in Table
\ref{t:regression_compar} which aims to compare different modifications of
SAT-RF. We can see that SAT-RF-yx outperforms other models for most dataset
when RFs are used. The same cannot be concluded when ERTs are used. Indeed, we
can see from Table \ref{t:regression_compar} that SAT-RF-yx based on ERTs is
inferior other model in half of cases.%

\begin{table}[tbp] \centering
\caption{Measure $R^2$  for comparison of models (SAT-RF-y, SAT-RF-x, SAT-RF-yx) trained on regression datasets}%
\begin{tabular}
[c]{ccccccc}\hline
& \multicolumn{3}{c}{RF} & \multicolumn{3}{c}{ERT}\\\hline
Data set & SAT-RF-y & SAT-RF-x & SAT-RF-yx & SAT-RF-y & SAT-RF-x &
SAT-RF-yx\\\hline
Diabetes & $\mathbf{0.424}$ & $0.416$ & $0.422$ & $\mathbf{0.453}$ & $0.442$ &
$0.439$\\\hline
Friedman 1 & $0.470$ & $0.470$ & $\mathbf{0.489}$ & $\mathbf{0.521}$ & $0.513$
& $0.513$\\\hline
Friedman 2 & $0.878$ & $0.872$ & $\mathbf{0.882}$ & $\mathbf{0.939}$ & $0.930$
& $0.930$\\\hline
Friedman 3 & $0.682$ & $0.684$ & $\mathbf{0.685}$ & $0.739$ & $0.739$ &
$\mathbf{0.751}$\\\hline
Regression & $0.454$ & $0.483$ & $\mathbf{0.488}$ & $0.455$ & $\mathbf{0.493}$
& $0.449$\\\hline
Sparse & $0.522$ & $0.488$ & $\mathbf{0.540}$ & $0.531$ & $0.518$ &
$\mathbf{0.542}$\\\hline
Airfoil & $0.823$ & $0.820$ & $\mathbf{0.849}$ & $0.802$ & $0.802$ &
$\mathbf{0.841}$\\\hline
Boston & $0.823$ & $0.821$ & $\mathbf{0.824}$ & $0.837$ & $0.843$ &
$\mathbf{0.844}$\\\hline
Concrete & $0.857$ & $0.857$ & $\mathbf{0.866}$ & $\mathbf{0.869}$ &
$\mathbf{0.863}$ & $0.860$\\\hline
Wine & $\mathbf{0.424}$ & $0.422$ & $\mathbf{0.429}$ & $0.419$ & $0.416$ &
$\mathbf{0.447}$\\\hline
Yacht & $\mathbf{0.989}$ & $\mathbf{0.989}$ & $\mathbf{0.989}$ &
$\mathbf{0.988}$ & $\mathbf{0.988}$ & $\mathbf{0.988}$\\\hline
\end{tabular}
\label{t:regression_compar}%
\end{table}%

It should be noted that SAT-RF is an extension of ABRF under condition that
the self-attention is supplemented. Therefore, we compare results of ABRF and
SAT-RF for two cases when RFs and ERTs are used. Moreover, we compare SAT-RF
with $\epsilon$-ABRF because SAT-RF is based on this modification of ABRF. The
best results of both models are shown in Table \ref{t:ABRF_compar}. It can be
seen from Table \ref{t:ABRF_compar} that adding the self-attention module to
ABRF improves the results. To formally test whether the average difference in
the performance of two models, SAT-RF and $\epsilon$-ABRF, is significantly
different from zero, we apply the $t$-test which has been proposed and
described by Demsar \cite{Demsar-2006}. Since we use differences between
accuracy measures of SAT-RF and $\epsilon$-ABRF, then they are compared with
$0$. The $t$ statistics in this case is distributed according to the Student
distribution with $11-1$ degrees of freedom. Results of computing the $t$
statistics of the difference are p-values denoted as $p$ and the $95\%$
confidence interval for the mean $0.0085$, which are $p=0.032$ and
$[0.00088,0.016]$, respectively. The $t$-test demonstrates the outperformance
of SAT-RF in comparison with $\epsilon$-ABRF for RFs because $p<0.05$. We also
compare the same models based on ERTs. We get the $95\%$ confidence interval
for the mean $0.0127$, which are $p=0.0112$ and $[0.0036,0219]$, respectively.%

\begin{table}[tbp] \centering
\caption{Measure $R^2$  for comparison of models SAT-RF and e-ABRF trained on regression datasets}%
\begin{tabular}
[c]{ccccc}\hline
& \multicolumn{2}{c}{RF} & \multicolumn{2}{c}{ERT}\\\hline
Data set & $\epsilon$-ABRF & SAT-RF & $\epsilon$-ABRF & SAT-RF\\\hline
Diabetes & $\mathbf{0.424}$ & $\mathbf{0.424}$ & $0.441$ & $\mathbf{0.453}%
$\\\hline
Friedman 1 & $0.470$ & $\mathbf{0.489}$ & $0.513$ & $\mathbf{0.521}$\\\hline
Friedman 2 & $0.877$ & $\mathbf{0.882}$ & $0.930$ & $\mathbf{0.939}$\\\hline
Friedman 3 & $\mathbf{0.686}$ & $0.685$ & $0.739$ & $\mathbf{0.751}$\\\hline
Regression & $0.450$ & $\mathbf{0.488}$ & $0.447$ & $\mathbf{0.493}$\\\hline
Sparse & $0.529$ & $\mathbf{0.540}$ & $0.536$ & $\mathbf{0.542}$\\\hline
Airfoil & $0.843$ & $\mathbf{0.849}$ & $0.837$ & $\mathbf{0.841}$\\\hline
Boston & $0.823$ & $\mathbf{0.824}$ & $0.838$ & $\mathbf{0.844}$\\\hline
Concrete & $0.857$ & $\mathbf{0.866}$ & $0.863$ & $\mathbf{0.869}$\\\hline
Wine & $0.423$ & $\mathbf{0.429}$ & $0.416$ & $\mathbf{0.447}$\\\hline
Yacht & $\mathbf{0.989}$ & $\mathbf{0.989}$ & $\mathbf{0.988}$ &
$\mathbf{0.988}$\\\hline
\end{tabular}
\label{t:ABRF_compar}%
\end{table}%

\section{Concluding remarks}

Extensions of the attention-based RF based on joint use of the attention and
self-attention mechanisms have been proposed. The attention part plays role of
assigning weights to decision trees in the RF, and the self-attention part
tries to capture dependencies of the tree predictions and to remove noise or
anomalous predictions. They can be regarded as an alternative tool for
handling tabular data. The proposed models allow us to avoid using neural
networks and gradient-based algorithms. One of the important peculiarities of
the models is that the attention and the self-attention are learned jointly by
solving the quadratic problem with the attention and self-attention weights as
optimization variables.

Advantages of the proposed models are the following. First, the models are
simply learned. Second, in contrast to neural networks, the models have a few
hyperparameters: parameters of the Gaussian kernels (softmax operations) and
the contamination parameters of the Huber's $\epsilon$-contamination model.
Third, the attention part allows us to improve predictions and the
self-attention part allows to take into account \textquotedblleft
bad\textquotedblright\ trees and anomalous predictions. Results of numerical
experiments clearly illustrate the above. The proposed models can be extended
by adding new self-attention modules which form the multi-head self-attention.
However, this extension is rather complex from the implementation point of view.

We have to point out also disadvantages. First, the proposed models are mainly
restricted by dealing with tabular data due to the RF basis of the models.
Second, in contrast to the attention-based RF, the model has a larger number
of training parameters (weights of the attention and the self-attention). If
the number of trees in the RF is rather large, then the number of training
parameters is significantly increases. This may lead to overfitting. Third,
the advantage of the models to handle tabular data can be viewed as its
disadvantage because other types of data, for instance, images, graphs, text
data may result worse predictions.

Many numerical experiments have demonstrated the outperformance of the
proposed models. Moreover, the results have demonstrated that SAT-RFs
outperform the attention-based RFs which are the basis for the proposed
models. Due to flexibility of SAT-RFs, many modifications can be proposed and
studied, for example, various kernel functions, models of weights different
from the Huber's $\epsilon$-contamination model. Attention weights as well as
self-attention weights can be assigned to subsets of trees. This approach
allows us to partially reduce the number of training parameters. It is
interesting to develop algorithms for implementing the multi-head
self-attention. All the above ideas can be regarded as direction for further research.

\bibliographystyle{unsrt}
\bibliography{Attention,Boosting,Classif_bib,Deep_Forest,Explain,Explain_med,Imprbib,Lasso,MIL,MYBIB,MYUSE}

\end{document}